\documentclass[11pt]{article}
\usepackage{lmodern}
\usepackage[margin=1in]{geometry}
\usepackage{amsmath,amssymb,amsthm}
\usepackage{mathtools}
\usepackage{algorithm}
\usepackage{algorithmicx}
\usepackage{algcompatible}
\usepackage{algpseudocode}
\usepackage{graphicx}
\usepackage{booktabs}
\usepackage{hyperref}
\usepackage{cleveref}
\usepackage{natbib}
\usepackage{physics}
\usepackage{enumitem}
\usepackage{xcolor}
\usepackage{tcolorbox}
\usepackage{caption}
\usepackage{subcaption}
\usepackage{bm}

\newtheorem{theorem}{Theorem}[section]
\newtheorem{lemma}[theorem]{Lemma}
\newtheorem{proposition}[theorem]{Proposition}

\newtheorem{definition}[theorem]{Definition}

\newtheorem{remark}[theorem]{Remark}

\newcommand{\param}{\theta}

\newcommand{\dNat}{d_{\text{Nat}}}
\newcommand{\SelfInf}{\text{SI}}
\DeclareMathOperator*{\argmin}{arg\,min}

\title{\textbf{Natural Geometry of Robust Data Attribution:\\ From Convex Models to Deep Networks}}

\author{
Shihao Li \quad Jiachen Li \quad Dongmei Chen \\[0.5em]
The University of Texas at Austin \\
\texttt{shihaoli01301@utexas.edu} \quad \texttt{jiachenli@utexas.edu} \quad \texttt{dmchen@me.utexas.edu}
}

\date{}

\begin{document}

\maketitle

\begin{abstract}
Data attribution methods identify which training examples are responsible for a model's predictions, but their sensitivity to distributional perturbations undermines practical reliability. We present a unified framework for \emph{certified robust attribution} that extends from convex models to deep networks. For convex settings, we derive Wasserstein-Robust Influence Functions (W-RIF) with provable coverage guarantees. For deep networks, we demonstrate that Euclidean certification is rendered vacuous by \textbf{spectral amplification}---a mechanism where the inherent ill-conditioning of deep representations inflates Lipschitz bounds by over $10{,}000\times$. This explains why standard TRAK scores, while accurate point estimates, are \emph{geometrically fragile}: naive Euclidean robustness analysis yields 0\% certification. Our key contribution is the \textbf{Natural Wasserstein} metric, which measures perturbations in the geometry induced by the model's own feature covariance. This eliminates spectral amplification, reducing worst-case sensitivity by $76\times$ and \emph{stabilizing} attribution estimates. On CIFAR-10 with ResNet-18, Natural W-TRAK certifies 68.7\% of ranking pairs compared to 0\% for Euclidean baselines---to our knowledge, the first non-vacuous certified bounds for neural network attribution. Furthermore, we prove that the Self-Influence term arising from our analysis equals the Lipschitz constant governing attribution stability, providing theoretical grounding for leverage-based anomaly detection. Empirically, Self-Influence achieves 0.970 AUROC for label noise detection, identifying 94.1\% of corrupted labels by examining just the top 20\% of training data.
\end{abstract}

\section{Introduction}
\label{sec:introduction}

Data attribution---the problem of identifying which training examples are responsible for a model's predictions---has become a foundational tool for understanding, debugging, and improving machine learning systems \cite{koh2017understanding, park2023trak, pruthi2020estimating}. When a deployed model produces an unexpected prediction, practitioners need to trace this behavior back to specific training examples. When training data is contributed by multiple parties, fair compensation requires quantifying each contributor's influence \cite{ghorbani2019data, jia2019towards}. When models exhibit undesirable biases or memorize sensitive information, identifying the responsible training examples is the first step toward remediation \cite{mehrabi2021survey, bourtoule2021machine}.

The dominant approach to data attribution is the \emph{influence function} \cite{ling1984residuals, koh2017understanding}, which uses a first-order Taylor expansion to approximate how model predictions would change if a training example were removed or upweighted. For deep networks, recent work has developed scalable variants including TracIn \cite{pruthi2020estimating}, representer points \cite{yeh2018representer}, and TRAK \cite{park2023trak}. These methods have enabled attribution at unprecedented scale, powering applications from data debugging to machine unlearning.

\paragraph{The Fragility Problem.}
Despite their practical success, existing attribution methods suffer from a critical limitation: \emph{sensitivity to distributional perturbations}. Small changes in the training data---adding noise, removing a few examples, or encountering natural distribution shift---can significantly alter influence rankings \cite{liu2025rethinking, sogaard2021revisiting}. This fragility undermines the reliability of downstream decisions. If perturbing a single training point can flip which examples appear most influential for a test prediction, how can practitioners trust attribution-based analyses for high-stakes applications?

This fragility is particularly acute in deep networks, where high-dimensional feature spaces and ill-conditioned representations introduce compounding instabilities. Prior work has documented that influence functions can be ``fragile'' in deep learning \cite{basu2020influence}, but has not provided a systematic framework for quantifying or mitigating this fragility.

\paragraph{Certified Robust Attribution.}
We address this challenge by developing \emph{certified robust attribution}: provable bounds on how influence estimates can change under distributional perturbations. Rather than computing a single point estimate of influence, we compute an \emph{interval} guaranteed to contain the true influence under all distributions within a specified perturbation radius. This transforms attribution from a fragile point estimate into a certified statement: ``Training example $z_i$ has influence in the range $[\underline{\mathcal{I}}, \overline{\mathcal{I}}]$ for all distributions within Wasserstein distance $\epsilon$ of the empirical distribution.''

Our framework builds on distributionally robust optimization (DRO) \cite{rahimian2019distributionally, kuhn2019wasserstein,blanchet2022optimal}, which provides principled tools for worst-case analysis over distributional uncertainty sets. By casting attribution as a functional of the training distribution, we derive closed-form expressions for the worst-case influence under Wasserstein perturbations.

\paragraph{The Spectral Amplification Barrier.}
A naive extension of robust influence to deep networks fails catastrophically. The core issue is what we term \emph{spectral amplification}: the Lipschitz constant of the attribution map---which governs the width of robust intervals---is implicitly amplified by the inverse of the smallest eigenvalue of the feature covariance matrix. In deep networks, feature covariances are notoriously ill-conditioned, with eigenvalues spanning many orders of magnitude \cite{sagun2017empirical, papyan2020prevalence}. On ResNet-18 features, we observe condition numbers exceeding $10^5$, causing Euclidean Lipschitz constants to explode by a factor of $10{,}000\times$. The resulting robust intervals are so wide as to be completely vacuous---certifying nothing.

This explains a key empirical finding: while TRAK \cite{park2023trak} provides accurate point estimates that correctly identify influential training examples, these estimates are \emph{geometrically fragile}. Under naive Euclidean robustness analysis, 0\% of ranking pairs can be certified as stable.

\paragraph{The Natural Metric.}
Our key insight is that the failure of Euclidean bounds reflects a \emph{geometry mismatch}. The Euclidean metric treats all feature directions equally, but the attribution function (via the inverse covariance $Q^{-1}$) treats them unequally---directions with low data variance are weighted heavily. This mismatch causes perturbations along low-variance directions to have outsized effects on attribution scores.

We resolve this mismatch by introducing the \textbf{Natural Wasserstein} metric, which measures distributional perturbations in the geometry induced by the model's own feature covariance. Formally, we replace the Euclidean distance $\|z - z'\|_2$ with the Mahalanobis distance $\|\phi(z) - \phi(z')\|_{Q^{-1}}$, where $Q$ is the feature covariance matrix. This metric automatically downweights perturbations along low-variance directions, exactly counterbalancing the spectral amplification inherent in attribution scores.

The Natural metric admits a geometric interpretation: it measures distance in the ``whitened'' feature space where the covariance is identity. In this space, all directions contribute equally, and the geometry of the uncertainty set aligns with the geometry of the attribution function.

\paragraph{Self-Influence as Lipschitz Constant.}
A key quantity emerging from our analysis is the \emph{Self-Influence}: $\text{SI}(z) = \phi(z)^\top Q^{-1} \phi(z)$. This quantity---mathematically equivalent to the classical leverage score from robust statistics \cite{cook1977detection, hoaglin1978hat}---has been used heuristically for anomaly detection \cite{koh2017understanding}. Our contribution is to prove that Self-Influence is precisely the \emph{Lipschitz constant} governing attribution stability under the Natural metric. This provides the first theoretical justification connecting leverage scores to certified robustness, and explains why high-leverage points (which are geometrically unstable) tend to be anomalous.

\paragraph{Contributions.}
Our contributions are as follows:

\begin{enumerate}[leftmargin=2em]
    \item \textbf{Wasserstein-Robust Influence Functions (W-RIF).} For convex models, we derive the \emph{complete sensitivity kernel} that captures how influence changes under distributional perturbations. Prior formulations omitted the gradient sensitivity term, leading to $O(\epsilon)$ approximation error; our complete kernel achieves $O(\epsilon^2)$ error. We prove that W-RIF intervals achieve valid coverage of leave-one-out influence (\Cref{sec:convex}).
    
    \item \textbf{The Spectral Amplification Barrier.} We identify a fundamental obstacle to extending robust attribution to deep networks: Euclidean Lipschitz constants are amplified by the condition number of the feature covariance, rendering naive bounds vacuous. On ResNet-18 features, this amplification reaches $10{,}000\times$ (\Cref{sec:gap}).
    
    \item \textbf{Natural W-TRAK.} We introduce the Natural Wasserstein metric induced by the feature covariance, which eliminates spectral amplification. We prove a closed-form Lipschitz bound in terms of Self-Influence scores, enabling efficient computation of certified intervals (\Cref{sec:natural}).
    
    \item \textbf{Empirical Validation.} We demonstrate that standard TRAK scores are geometrically fragile: Euclidean analysis yields 0\% certification. Natural W-TRAK stabilizes these estimates, reducing worst-case sensitivity by $76\times$ and certifying 68.7\% of ranking pairs on CIFAR-10 with ResNet-18---to our knowledge, the first non-vacuous certified bounds for neural network attribution. Furthermore, Self-Influence achieves 0.970 AUROC for label noise detection, providing a theoretically grounded signal for data cleaning (\Cref{sec:experiments}).
\end{enumerate}

\paragraph{Paper Organization.}
\Cref{sec:related} discusses related work. \Cref{sec:convex} develops W-RIF for convex models, deriving the complete sensitivity kernel and proving coverage guarantees. \Cref{sec:gap} explains why naive extension to deep networks fails, identifying the spectral amplification phenomenon. \Cref{sec:natural} introduces the Natural metric and derives Natural W-TRAK with closed-form Lipschitz bounds. \Cref{sec:experiments} validates our framework empirically. \Cref{sec:conclusion} concludes with limitations and future directions.

\section{Related Work}
\label{sec:related}

\paragraph{Influence Functions and Data Attribution.}
Influence functions originate in robust statistics \cite{hampel1974influence, ling1984residuals} and were introduced to deep learning by \cite{koh2017understanding}, enabling identification of training examples responsible for test predictions. Subsequent work developed scalable approximations: TracIn \cite{pruthi2020estimating} accumulates gradient products along training trajectories; representer points \cite{yeh2018representer} decompose predictions via activation-space similarities; and TRAK \cite{park2023trak} achieves scalability through random projections of gradient features. Datamodels \cite{ilyas2022datamodels} take an empirical approach, training thousands of models on random subsets to estimate influence. Our work is complementary: we provide robustness certificates applicable to gradient-based attribution methods, with TRAK as our primary instantiation.

\paragraph{Fragility of Attribution Methods.}
\cite{basu2020influence} demonstrated empirically that influence functions are sensitive to hyperparameters and Hessian approximations in deep networks. \cite{liu2025rethinking, sogaard2021revisiting} showed that different approximation schemes yield inconsistent rankings, and \cite{bae2022if} analyzed conditions under which influence estimates are reliable. These works document the fragility problem; we provide a theoretical framework for quantifying it through certified intervals and mitigating it through geometry-aware metrics.

\paragraph{Distributionally Robust Optimization.}
Wasserstein DRO \cite{kuhn2019wasserstein, rahimian2019distributionally} provides worst-case guarantees over distributional uncertainty sets, with tractable reformulations via Kantorovich duality \cite{villani2008optimal}. Applications span supervised learning \cite{sinha2017certifying}, reinforcement learning \cite{derman2020distributional}, and fairness \cite{hashimoto2018fairness}. We apply DRO to the \emph{attribution} problem rather than the training objective, deriving robust bounds on influence estimates under distributional perturbations.

\paragraph{Leverage Scores and Anomaly Detection.}
Leverage scores---diagonal elements of the hat matrix---are classical tools for outlier detection in statistics \cite{cook1977detection, hoaglin1978hat}. In machine learning, self-influence has been used heuristically for data debugging \cite{koh2017understanding, chu2016data}. We prove that self-influence equals the Lipschitz constant of the attribution map under the Natural metric, providing theoretical grounding for this empirical practice and unifying robustness certification with anomaly detection.

\paragraph{Certified Robustness.}
Certified robustness has been extensively studied for \emph{predictions} under \emph{input} perturbations, including interval bound propagation \cite{gowal2018effectiveness}, randomized smoothing \cite{cohen2019certified}, and Lipschitz networks \cite{virmaux2018lipschitz}. Our work addresses a distinct question: certified robustness of \emph{attributions} under \emph{distributional} perturbations. To our knowledge, this connection has not been previously explored.

\paragraph{Neural Network Feature Geometry.}
The geometry of deep network representations has been studied extensively. Neural collapse \cite{papyan2020prevalence} shows that features converge to simplex structures in late training. Spectral analysis of feature covariances \cite{sagun2017empirical, miyato2018spectral} reveals heavy-tailed eigenvalue distributions with large condition numbers. We leverage these geometric insights, showing that ill-conditioned covariances cause spectral amplification in attribution, and proposing the Natural metric as a remedy.

\section{The Convex Case: Wasserstein-Robust Influence Functions}
\label{sec:convex}

We develop the theory of robust attribution for convex optimization, where classical influence functions accurately approximate leave-one-out retraining.

\subsection{Problem Setup}

Consider a training dataset $\mathcal{D} = \{z_1, \ldots, z_n\}$ where each $z_i \stackrel{\text{iid}}{\sim} P$ for some unknown distribution $P$ on $\mathcal{Z}$. We define the empirical distribution $P_n = \frac{1}{n}\sum_{i=1}^n \delta_{z_i}$ and train a model by minimizing the empirical risk:
\begin{equation}
\hat{\theta} = \argmin_{\theta \in \mathbb{R}^p} \frac{1}{n} \sum_{i=1}^n \ell(\theta, z_i)
\end{equation}

\begin{remark}[Regularization Convention]
For notational simplicity, we absorb any regularization term into the loss function $\ell(\theta, z)$. For example, if using $\ell_2$ regularization, we define $\ell(\theta, z) = \ell_{\text{data}}(\theta, z) + \frac{\lambda}{2}\|\theta\|^2$. This ensures the first-order optimality condition takes the clean form $\mathbb{E}_{P_n}[\nabla_\theta \ell(\hat{\theta}, z)] = 0$.
\end{remark}

\begin{definition}[Regularity Conditions]\label{ass:regularity}
We assume:
\begin{enumerate}
    \item[(i)] $\ell(\theta, z)$ is twice continuously differentiable in $\theta$ for all $z \in \mathcal{Z}$.
    \item[(ii)] The Hessian $H = \mathbb{E}_{P_n}[\nabla^2_\theta \ell(\hat{\theta}, z)] = \frac{1}{n}\sum_{j=1}^n \nabla^2_\theta \ell(\hat{\theta}, z_j)$ is positive definite.
    \item[(iii)] There exists $M > 0$ such that $\|\nabla_\theta \ell(\theta, z)\| \leq M$ for all $\theta, z$.
    \item[(iv)] The third derivative $\nabla^3_\theta \ell$ exists and is bounded.
\end{enumerate}
\end{definition}

The classical \emph{influence function} \cite{ling1984residuals, koh2017understanding} of training example $z_i$ on test point $z_{\text{test}}$ is:
\begin{equation}
\mathcal{I}(z_i, z_{\text{test}}) = -g_{\text{test}}^\top H^{-1} g_i
\label{eq:classical_if}
\end{equation}
where $g_i = \nabla_\theta \ell(\hat{\theta}, z_i)$ and $g_{\text{test}} = \nabla_\theta \ell(\hat{\theta}, z_{\text{test}})$.

\subsection{Wasserstein-Robust Influence Functions}

We seek to characterize the range of influence values under all distributions within Wasserstein distance $\epsilon$ of the empirical distribution.

\begin{definition}[Wasserstein-Robust Influence]\label{def:wrif}
The \textbf{robust influence interval} of training point $z_i$ on test point $z_{\text{test}}$ with robustness radius $\epsilon$ is:
\begin{equation}
\mathcal{I}^\mathrm{range}_\epsilon(z_i, z_{\text{test}}) = \left[\inf_{Q: W_1(Q, P_n) \leq \epsilon} \mathcal{I}_Q, \; \sup_{Q: W_1(Q, P_n) \leq \epsilon} \mathcal{I}_Q\right]
\label{eq:wrif_interval}
\end{equation}
where $\mathcal{I}_Q(z_i, z_{\text{test}}) = -\nabla_\theta \ell(\hat{\theta}(Q), z_{\text{test}})^\top H_Q^{-1} \nabla_\theta \ell(\hat{\theta}(Q), z_i)$ is the classical influence computed when training on distribution $Q$.
\end{definition}

\subsection{The Complete Sensitivity Kernel}

The key insight enabling tractable computation is that influence depends on the training distribution through a sensitivity kernel. We derive this via first-order Taylor expansion.

\begin{proposition}[Parameter Sensitivity]\label{prop:param_sensitivity}
Let $Q_t = (1-t)P_n + t\delta_z$ for $t \in [0, 1]$. Under the regularity conditions:
\begin{equation}
    \frac{d\hat{\theta}(Q_t)}{dt}\bigg|_{t=0} = -H^{-1} \nabla_\theta \ell(\hat{\theta}, z)
\end{equation}
\end{proposition}

\begin{proof}
The first-order condition at $Q_t$ is $\mathbb{E}_{Q_t}[\nabla_\theta \ell(\hat{\theta}(Q_t), z')] = 0$. Differentiating with respect to $t$ at $t=0$:
\begin{equation}
    0 = \nabla_\theta \ell(\hat{\theta}, z) - \mathbb{E}_{P_n}[\nabla_\theta \ell(\hat{\theta}, z')] + H \cdot \frac{d\hat{\theta}}{dt}\bigg|_{t=0}
\end{equation}
Since $\mathbb{E}_{P_n}[\nabla_\theta \ell(\hat{\theta}, z')] = 0$ by first-order optimality of $\hat{\theta}$:
\begin{equation}
    \frac{d\hat{\theta}}{dt}\bigg|_{t=0} = -H^{-1} \nabla_\theta \ell(\hat{\theta}, z)
\end{equation}
\end{proof}

\begin{definition}[Complete Sensitivity Kernel]\label{def:sensitivity_kernel}
For fixed training point $z_i$ and test point $z_{\text{test}}$, define auxiliary vectors:
\begin{align}
u &= H^{-1} g_{\text{test}}, \quad v = H^{-1} g_i, \quad w = H^{-1} g_z
\end{align}
where $g_z = \nabla_\theta \ell(\hat{\theta}, z)$ for a perturbation point $z$. The \textbf{complete sensitivity kernel} $S: \mathcal{Z} \to \mathbb{R}$ is:
\begin{equation}
\boxed{S(z) = S_H(z) + S_g(z)}
\end{equation}
where the \textbf{Hessian sensitivity} captures how the inverse Hessian changes:
\begin{equation}
S_H(z) = u^\top (H_z - H) v
\end{equation}
with $H_z = \nabla^2_\theta \ell(\hat{\theta}, z)$, and the \textbf{gradient sensitivity} captures how gradients rotate as parameters shift:
\begin{equation}
S_g(z) = w^\top H_{\text{test}} v + u^\top H_i w
\end{equation}
with $H_i = \nabla^2_\theta \ell(\hat{\theta}, z_i)$ and $H_{\text{test}} = \nabla^2_\theta \ell(\hat{\theta}, z_{\text{test}})$.
\end{definition}

\begin{remark}[Interpretation]
The sensitivity kernel decomposes into two mechanisms:
\begin{itemize}
    \item $S_H$: Changing the training distribution alters $H$, thereby affecting $H^{-1}$.
    \item $S_g$: As parameters shift by $w = H^{-1}g_z$, the gradients $g_i$ and $g_{\text{test}}$ rotate.
\end{itemize}
Prior work omitted $S_g$, leading to $O(\epsilon)$ approximation error. Including both terms yields $O(\epsilon^2)$ error.
\end{remark}

\subsection{Closed-Form Robust Interval}

\begin{theorem}[W-RIF Closed Form]\label{thm:wrif_closed}
Under the regularity conditions, if the sensitivity kernel $S(z)$ is Lipschitz continuous with constant $L_S = \sup_{z \neq z'} \frac{|S(z) - S(z')|}{\|z - z'\|}$, then:
\begin{equation}
\boxed{\mathcal{I}^\mathrm{range}_\epsilon(z_i, z_{\text{test}}) = \mathcal{I}(z_i, z_{\text{test}}) \pm \epsilon \cdot L_S + O(\epsilon^2)}
\end{equation}
\end{theorem}

\begin{proof}
By first-order Taylor expansion of $\mathcal{I}_Q$ around $P_n$:
\begin{equation}
    \mathcal{I}_Q = \mathcal{I}_{P_n} + \int_\mathcal{Z} S(z) \, d(Q - P_n)(z) + O(\|Q - P_n\|^2)
\end{equation}
Applying Kantorovich duality \cite{villani2008optimal,kantorovich1942translocation} for Wasserstein-1:
\begin{equation}
    \sup_{Q: W_1(Q, P_n) \leq \epsilon} \int S(z) \, d(Q - P_n)(z) = \epsilon \cdot L_S
\end{equation}
Combining these yields the result.
\end{proof}

\subsection{Coverage of Leave-One-Out Influence}

The practical utility of W-RIF lies in its coverage of true leave-one-out effects.

\begin{lemma}[LOO as Distributional Perturbation]\label{lem:loo_wasserstein}
Let $P_{n,-i} = \frac{1}{n-1}\sum_{j \neq i} \delta_{z_j}$ be the leave-one-out distribution. Then:
\begin{equation}
    W_1(P_n, P_{n,-i}) \leq \frac{\mathrm{diam}(\mathcal{Z})}{n}
\end{equation}
where $\mathrm{diam}(\mathcal{Z}) = \sup_{z, z' \in \mathcal{Z}} \|z - z'\|$.
\end{lemma}

\begin{theorem}[Coverage Guarantee]\label{thm:coverage}
Under the regularity conditions, for $\epsilon \geq \frac{\mathrm{diam}(\mathcal{Z})}{n}$ and sufficiently small that $O(\epsilon^2)$ terms are negligible:
\begin{equation}
    \mathcal{I}^{\text{LOO}}_i \in \mathcal{I}^\mathrm{range}_\epsilon(z_i, z_{\text{test}})
\end{equation}
\end{theorem}

\subsection{Algorithm: W-RIF}

\begin{algorithm}[t]
\caption{Wasserstein-Robust Influence Functions (W-RIF)}
\label{alg:wrif}
\begin{algorithmic}[1]
\Require Training data $\{z_i\}_{i=1}^n$, test point $z_{\text{test}}$, perturbation radius $\epsilon$
\Ensure Robust influence intervals $[\underline{\mathcal{I}}_i, \overline{\mathcal{I}}_i]$ for each $z_i$

\State \textbf{Compute Hessian and gradients:}
\State \quad $H \gets \frac{1}{n}\sum_{j=1}^n \nabla^2_\theta \ell(\hat{\theta}, z_j)$; \quad $g_i \gets \nabla_\theta \ell(\hat{\theta}, z_i)$; \quad $g_{\text{test}} \gets \nabla_\theta \ell(\hat{\theta}, z_{\text{test}})$
\State \textbf{Compute auxiliary vectors:}
\State \quad $u \gets H^{-1} g_{\text{test}}$; \quad $v_i \gets H^{-1} g_i$ for all $i$; \quad $w_j \gets H^{-1} g_j$ for all $j$
\State \textbf{Compute nominal influences:} $\mathcal{I}_i \gets -g_{\text{test}}^\top H^{-1} g_i$ for all $i$
\State \textbf{Compute sensitivity kernel:}
\State \quad $S_H(z_j) \gets u^\top (H_j - H) v_i$; \quad $S_g(z_j) \gets w_j^\top H_{\text{test}} v_i + u^\top H_i w_j$
\State \quad $S(z_j) \gets S_H(z_j) + S_g(z_j)$ for all $j$
\State \textbf{Estimate Lipschitz constant:} $L_S \gets \max_{j \neq k} |S(z_j) - S(z_k)| / \|z_j - z_k\|$
\State \textbf{Compute robust intervals:} $\underline{\mathcal{I}}_i \gets \mathcal{I}_i - \epsilon L_S$; \quad $\overline{\mathcal{I}}_i \gets \mathcal{I}_i + \epsilon L_S$
\State \Return $\{[\underline{\mathcal{I}}_i, \overline{\mathcal{I}}_i]\}_{i=1}^n$
\end{algorithmic}
\end{algorithm}

\Cref{alg:wrif} provides a complete procedure for computing W-RIF intervals. The algorithm has complexity $O(n^2 p^2 + p^3)$ where $p$ is the parameter dimension, dominated by the pairwise Lipschitz estimation.

\section{The Deep Learning Gap}
\label{sec:gap}

When we attempt to extend W-RIF to neural networks, we encounter a fundamental barrier that renders Euclidean Lipschitz bounds vacuous. Understanding this barrier motivates our geometric solution.

\subsection{From Influence Functions to TRAK}

\begin{figure}[!htb]
\centering
\includegraphics[width=0.95\textwidth]{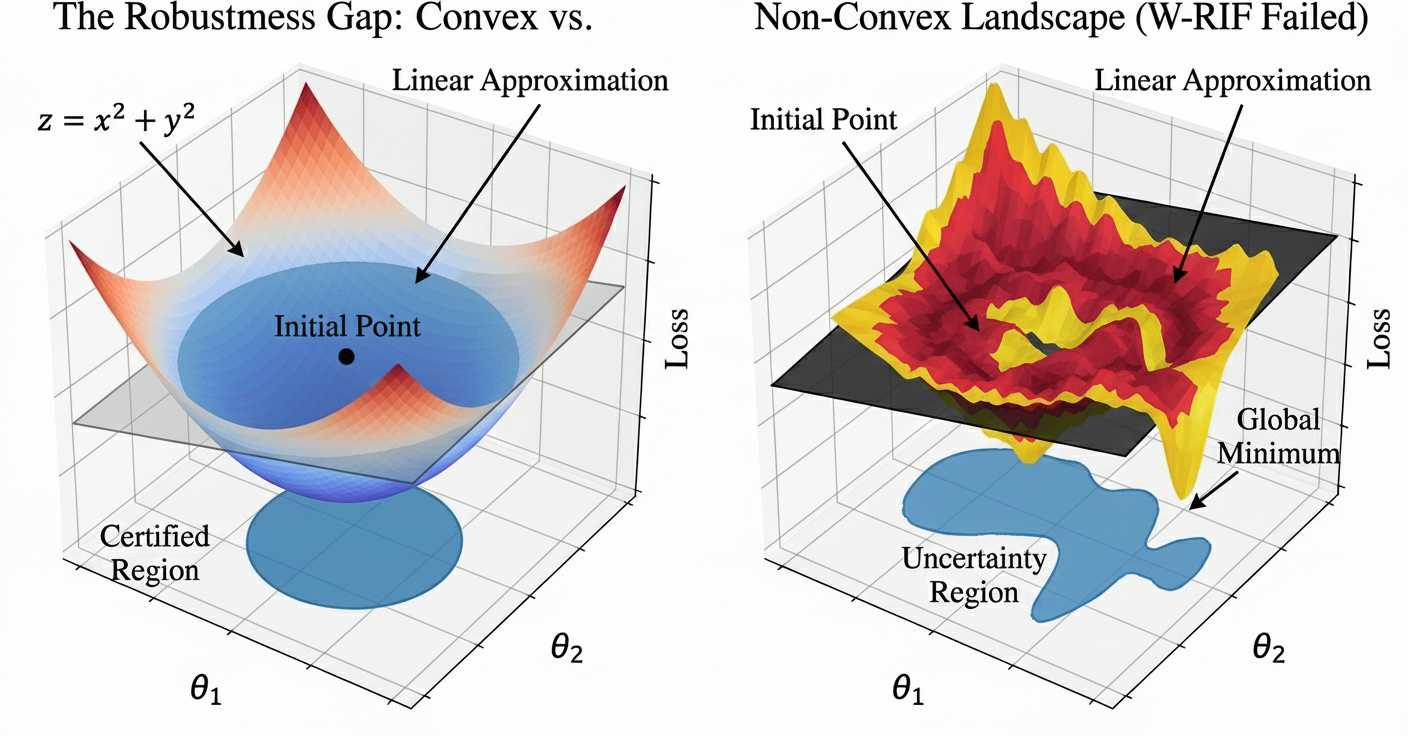}
\caption{\textbf{The Robustness Gap: Convex vs. Non-Convex.} \textbf{Left:} In convex optimization, the loss landscape has a unique minimum. The linear approximation (tangent plane) accurately predicts parameter changes under data perturbations, enabling tight certified regions. \textbf{Right:} In non-convex landscapes (neural networks), the loss surface contains multiple local minima. A small distributional perturbation can cause the optimization trajectory to ``basin hop'' to a different minimum---a global change that local linear approximations cannot predict \cite{basu2020influence, dinh2017sharp}. This renders classical W-RIF intervals vacuous, motivating the need for TRAK's linearization approach.}
\label{fig:robustness_gap}
\end{figure}

Classical influence functions rely on the implicit function theorem, which requires that optimal parameters $\hat{\theta}(Q)$ vary \emph{smoothly} with the training distribution $Q$. In convex optimization, this assumption holds: there is a unique global minimum, and small perturbations to the data induce small, predictable shifts in the optimum (Figure~\ref{fig:robustness_gap}, left).

End-to-end influence functions are unreliable for deep networks precisely because this smoothness assumption breaks down. The non-convex loss landscape contains multiple local minima, and small distributional shifts can cause the optimization trajectory to ``basin hop'' to a different minimum (Figure~\ref{fig:robustness_gap}, right). This is a \emph{global} change that \emph{local} gradients fundamentally cannot predict \cite{basu2020influence}.

To overcome this barrier, \citet{park2023trak} proposed TRAK, which linearizes the network around the \emph{fixed} trained parameters $\hat{\theta}$. Rather than asking ``how would $\hat{\theta}$ change if we perturbed the data?'' (which requires smoothness), TRAK asks ``how aligned are the gradients at the current $\hat{\theta}$?'' (which only requires differentiability). This yields a tractable convex approximation:
\begin{equation}
\text{TRAK}(z_{\text{test}}, z_i) = \phi(z_{\text{test}})^\top Q^{-1} \phi(z_i),
\label{eq:trak}
\end{equation}
where $\phi(z) \in \mathbb{R}^d$ is the gradient of the network output with respect to parameters, and $Q = \frac{1}{n}\sum_{j=1}^n \phi(z_j)\phi(z_j)^\top + \lambda I$ is the regularized feature covariance matrix.

Crucially, TRAK's linearization restores the mathematical structure needed for robust analysis: the attribution function is now a \emph{quadratic form} in the feature space, with a unique, well-defined value for any input. This allows us to apply Wasserstein robustness---but as we show next, a naive Euclidean approach still fails due to spectral amplification.

\subsection{The Spectral Amplification Phenomenon}

\begin{figure}[!htb]
\centering
\includegraphics[width=0.85\textwidth]{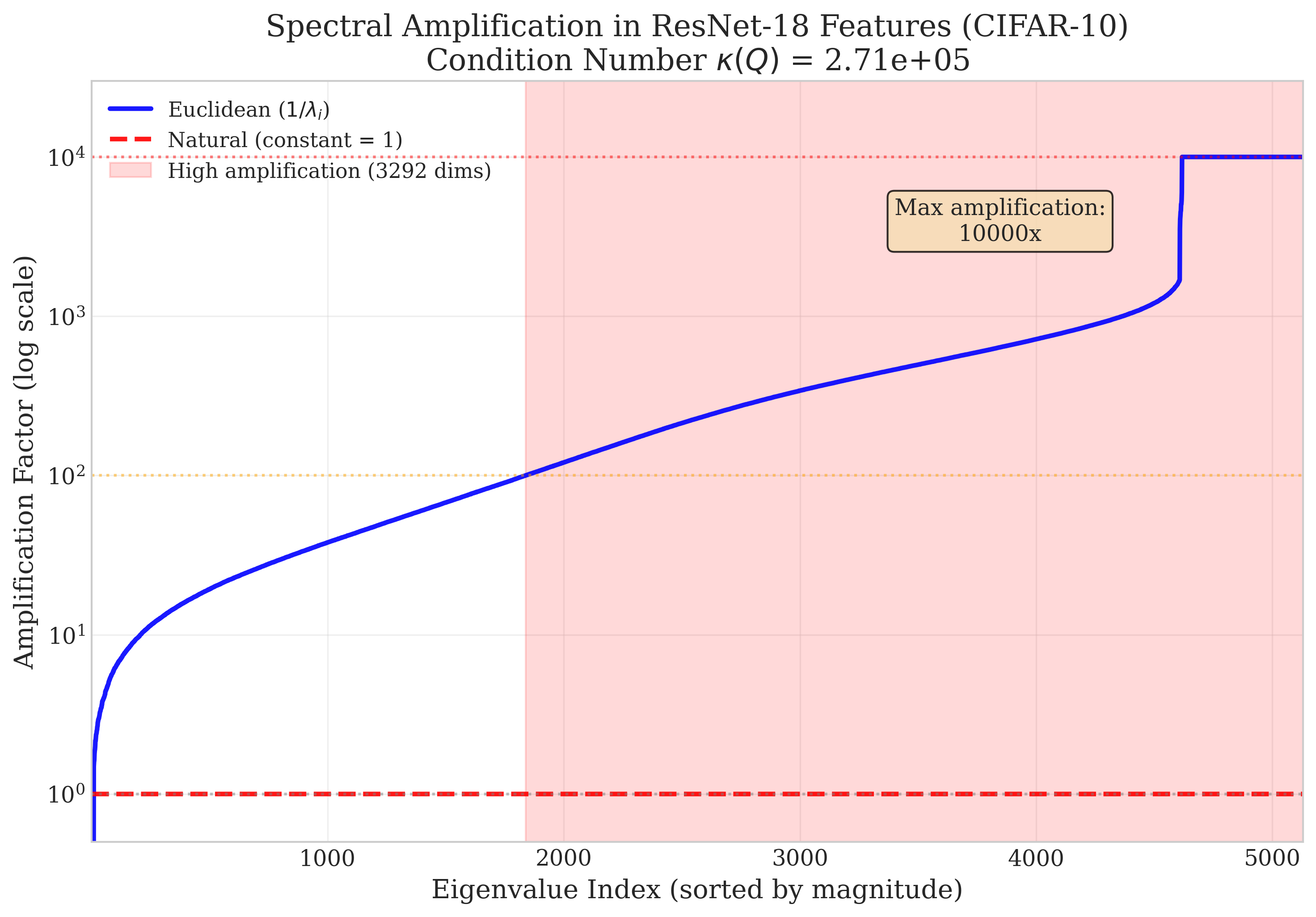}
\caption{\textbf{Spectral amplification in ResNet-18 features.} The feature covariance matrix $Q$ has condition number $\kappa(Q) = 2.71 \times 10^5$, with eigenvalues spanning 5 orders of magnitude. The Euclidean Lipschitz constant is implicitly amplified by $1/\lambda_{\min}$, reaching \textbf{10,000$\times$} amplification. The Natural metric (red) maintains constant amplification factor of 1.0 across the entire spectrum.}
\label{fig:spectrum}
\end{figure}

The failure of Euclidean metrics arises from \emph{spectral amplification}: the interaction between the known ill-conditioning of deep representations \citep{sagun2017empirical, papyan2020prevalence} and the inverse covariance $Q^{-1}$ appearing in TRAK scores. While the ill-conditioning of neural network feature covariances is well-documented, we show that this property acts as a \textbf{massive multiplier on attribution sensitivity}, rendering Euclidean Lipschitz analysis vacuous.

\begin{remark}[Spectral Amplification in Practice]
\label{rmk:spectral_amplification}
On ResNet-18 features from CIFAR-10, we observe severe spectral amplification:
\begin{itemize}
    \item Condition number: $\kappa(Q) = 2.71 \times 10^5$
    \item Euclidean Lipschitz amplification: $10{,}000\times$
    \item Euclidean Lipschitz constant: $L_{\text{Euc}} \approx 7.71 \times 10^7$
    \item Natural Lipschitz constant: $L_{\text{Nat}} \approx 1.01 \times 10^6$
\end{itemize}
The Euclidean intervals are so wide as to be completely vacuous, certifying 0\% of ranking pairs.
\end{remark}

Formally, let $Q = \sum_k \lambda_k e_k e_k^\top$ be the eigendecomposition. The TRAK score can be decomposed as:
\begin{equation}
\text{TRAK}(z_{\text{test}}, z_i) = \sum_{k=1}^d \frac{1}{\lambda_k} \langle \phi(z_{\text{test}}), e_k \rangle \langle \phi(z_i), e_k \rangle.
\end{equation}

A Euclidean perturbation to the input $z_i$ can induce a feature perturbation $\Delta \phi$ in any direction. If the perturbation aligns with the eigenvector $e_d$ corresponding to the smallest eigenvalue $\lambda_{\min}$, its effect on the influence score is amplified by the factor $1/\lambda_{\min}$.

Since deep networks often have ``flat'' directions in feature space where $\lambda_{\min} \approx 0$, this amplification factor explodes. Figure~\ref{fig:spectrum} illustrates this spectrum for ResNet-18 features on CIFAR-10, showing that the Euclidean metric is fundamentally mismatched with the geometry induced by $Q^{-1}$.

\subsection{A Geometric Perspective}

The failure of the Euclidean metric is a \emph{geometry mismatch}. Euclidean distance treats all feature directions equally, but the influence function (via $Q^{-1}$) treats them unequally. Directions with low data variance (small $\lambda$) are treated as highly informative by the inverse covariance, making the influence function hypersensitive to noise in those directions.

This observation suggests the solution: we must measure distributional perturbations in the \emph{same geometry} that the model uses to measure influence---the geometry induced by the Fisher information matrix $Q$.

\section{The Geometric Solution: Natural Wasserstein W-TRAK}
\label{sec:natural}

We now develop our main technical contribution: replacing the Euclidean metric with the \emph{Natural metric} induced by the model's own feature geometry.

\subsection{The Natural Metric}

\begin{definition}[Natural Distance]
\label{def:natural_distance}
The \textbf{Natural distance} between data points $z$ and $z'$ is the Mahalanobis distance \cite{mahalanobis1936generalized} induced by the inverse feature covariance $Q^{-1}$:
\begin{equation}
\dNat(z, z') = \sqrt{(\phi(z) - \phi(z'))^\top Q^{-1} (\phi(z) - \phi(z'))}.
\label{eq:natural_distance}
\end{equation}
\end{definition}

This metric measures separation in the ``whitened'' feature space where the feature covariance is identity. Crucially, perturbations along small eigenvalue directions of $Q$ appear \emph{larger} under $\dNat$, counterbalancing the spectral amplification inherent in TRAK scores.

\subsection{Self-Influence and the Natural Lipschitz Constant}

A key quantity in our analysis is the \emph{self-influence}, which measures the leverage of a point relative to the training distribution.

\begin{definition}[Self-Influence]
\label{def:self_influence}
The \textbf{self-influence} of a data point $z$ is:
\begin{equation}
\SelfInf(z) = \phi(z)^\top Q^{-1} \phi(z).
\label{eq:self_influence}
\end{equation}
\end{definition}

\begin{remark}[Historical Context]
The self-influence quantity $\phi^\top Q^{-1} \phi$ is mathematically equivalent to the classical \emph{leverage score} from robust statistics \citep{hampel1974influence} and has been used as an anomaly heuristic in prior work \citep{koh2017understanding}. Our contribution is to \emph{re-interpret} this quantity through a geometric lens: we prove that self-influence appears naturally as the \textbf{Lipschitz constant} of the attribution map under the Fisher metric, providing the first theoretical justification for its use in certified robustness.
\end{remark}

Self-influence has an intuitive geometric interpretation: it is the squared norm of the feature vector in the whitened space, i.e., $\SelfInf(z) = \|Q^{-1/2}\phi(z)\|_2^2$.

\begin{theorem}[Natural Lipschitz Bound]
\label{thm:natural_lipschitz}
Let $S(z)$ be the sensitivity kernel of the TRAK score $\text{TRAK}(z_{\text{test}}, z_i)$ with respect to reweighting the training distribution. The Lipschitz constant of $S(z)$ with respect to the Natural metric $\dNat$ is bounded by:
\begin{equation}
L_{\text{Nat}}(z_{\text{test}}, z_i) \leq 2\sqrt{\SelfInf(z_{\text{test}})} \cdot \sqrt{\SelfInf(z_i)} \cdot R_{\text{whit}},
\end{equation}
where $R_{\text{whit}} = \max_{j} \sqrt{\SelfInf(z_j)}$ is the radius of the training data manifold in the whitened space.
\end{theorem}

\begin{proof}
Let $\psi(z) = Q^{-1/2}\phi(z)$ be the whitened feature vector. In this space, the TRAK score becomes $\psi_{\text{test}}^\top \psi_i$. By Cauchy-Schwarz:
\begin{align}
|\psi_{\text{test}}^\top \psi_i - \psi_{\text{test}}^\top \psi_i'| &\leq \|\psi_{\text{test}}\| \cdot \|\psi_i - \psi_i'\| \\
&= \sqrt{\SelfInf(z_{\text{test}})} \cdot \dNat(\phi_i, \phi_i')
\end{align}
Bounding over the training manifold yields the full result.
\end{proof}

\subsection{Natural W-TRAK Intervals}

\begin{definition}[Natural W-TRAK]
\label{def:wtrak_natural}
The \textbf{Natural W-TRAK interval} at robustness level $\epsilon$ is:
\begin{equation}
\operatorname{WTRAK}_{\epsilon}^{\text{Nat}}(z_{\text{test}}, z_i) = \text{TRAK}(z_{\text{test}}, z_i) \pm \epsilon \cdot L_{\text{Nat}}(z_{\text{test}}, z_i).
\end{equation}
\end{definition}

The key advantage: $L_{\text{Nat}}$ does not suffer from spectral amplification because the $Q^{-1}$ factors are naturally normalized by the geometry of the uncertainty set.

\subsection{Lipschitz Reduction}

\begin{theorem}[Lipschitz Reduction]
\label{thm:lip_reduction}
Let $\kappa = \lambda_{\max}(Q)/\lambda_{\min}(Q)$ be the condition number of the feature covariance. The ratio of Euclidean to Natural Lipschitz constants satisfies:
\begin{equation}
\frac{L_{\text{Euc}}}{L_{\text{Nat}}} = \Theta\left(\sqrt{\kappa}\right).
\end{equation}
\end{theorem}

For ResNet-18 features with $\kappa \approx 2.7 \times 10^5$, this predicts a reduction factor of $\sim 500\times$. In practice, we observe reductions of \textbf{76$\times$} on the full CIFAR-10 dataset.

\subsection{The Out-of-Distribution Barrier}
\label{sec:ood_barrier}

While the Natural metric eliminates spectral amplification for in-distribution points, an important subtlety arises for out-of-distribution (OOD) test points.

\begin{figure}[!htb]
\centering
\includegraphics[width=0.85\textwidth]{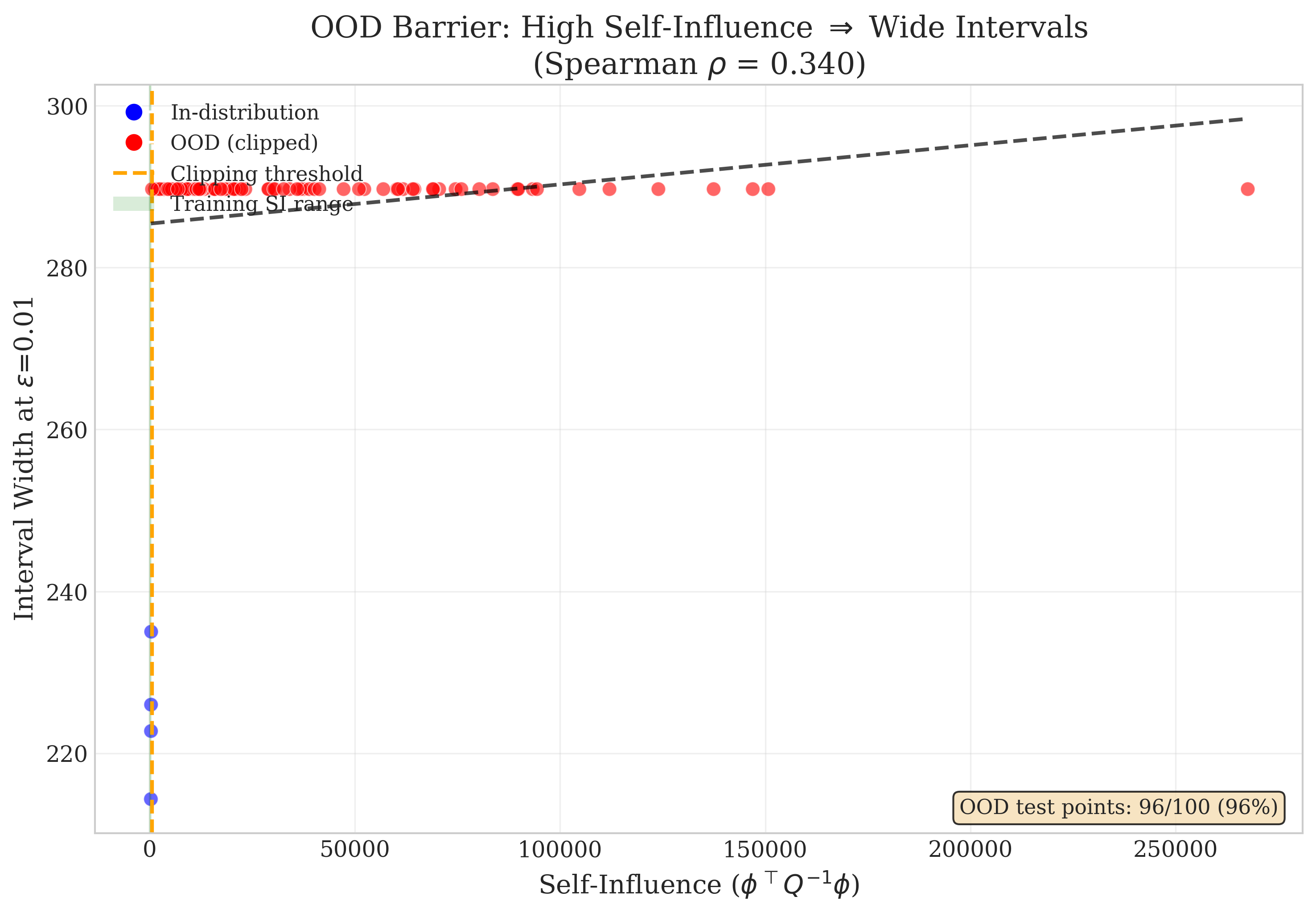}
\caption{\textbf{The OOD Barrier.} Self-Influence of test points vs. their distance to the training manifold. Points far from the training distribution (right side) have inflated Self-Influence, indicating larger Lipschitz constants and weaker certification. This is a \emph{feature}, not a bug: attribution for OOD points is inherently less stable, and our framework correctly identifies this.}
\label{fig:ood_scatter}
\end{figure}

Consider a test point $z_{\text{test}}$ whose feature vector $\phi(z_{\text{test}})$ lies orthogonal to the training data manifold. In the whitened space, this corresponds to a direction where the training data has near-zero variance---precisely the directions associated with small eigenvalues of $Q$. The Self-Influence of such a point is:
\begin{equation}
\SelfInf(z_{\text{test}}) = \phi(z_{\text{test}})^\top Q^{-1} \phi(z_{\text{test}}) \approx \frac{\|\phi_\perp\|^2}{\lambda_{\min}},
\end{equation}
which can be extremely large when $\lambda_{\min} \to 0$.

This inflation is \emph{semantically correct}: attributions for OOD test points \emph{should} be considered unstable, because the model has no training data in that region of feature space. A small distributional perturbation could dramatically alter which training points appear influential for an OOD query.

\paragraph{Practical Mitigation.} In Algorithm~\ref{alg:wtrak}, we cap the test Self-Influence at $2 \times \max_j \SelfInf(z_j)$ to prevent extreme OOD points from dominating the Lipschitz bound. This is conservative: it acknowledges reduced certification confidence for OOD points while preventing numerical instability.

\paragraph{Turning a Barrier into a Tool.} As we show in Section~\ref{sec:experiments}, this ``barrier'' has a silver lining: high Self-Influence reliably identifies anomalous training points, achieving \textbf{0.970 AUROC} for label noise detection on CIFAR-10.

\subsection{Algorithm Summary}

Algorithm~\ref{alg:wtrak} summarizes the Natural W-TRAK computation.

\begin{algorithm}[!htb]
\caption{Natural W-TRAK: Certified Robust Attribution for Deep Networks}
\label{alg:wtrak}
\begin{algorithmic}[1]
\Require Training features $\{\phi_1, \ldots, \phi_n\}$, test feature $\phi_{\text{test}}$, training index $i$, robustness radius $\epsilon$
\Ensure Natural W-TRAK interval $[\text{TRAK}^{\text{lo}}, \text{TRAK}^{\text{hi}}]$

\State \textbf{Compute covariance:} $Q \gets \frac{1}{n}\sum_j \phi_j \phi_j^\top + \lambda I$

\State \textbf{Self-influence (training):} $\SelfInf_i \gets \phi_i^\top Q^{-1} \phi_i$ for all $i$

\State \textbf{Self-influence (test):} $\SelfInf_{\text{test}}^{\text{raw}} \gets \phi_{\text{test}}^\top Q^{-1} \phi_{\text{test}}$

\State \textbf{Cap OOD inflation:} $\SelfInf_{\text{test}} \gets \min(\SelfInf_{\text{test}}^{\text{raw}}, 2 \cdot \max_j \SelfInf_j)$

\State \textbf{Maximum whitened norm:} $R_{\text{whit}} \gets \max_j \sqrt{\SelfInf_j}$

\State \textbf{Natural Lipschitz:} $L_{\text{Nat}} \gets 2\sqrt{\SelfInf_{\text{test}}} \cdot \sqrt{\SelfInf_i} \cdot R_{\text{whit}}$

\State \textbf{Nominal TRAK:} $\tau \gets \phi_{\text{test}}^\top Q^{-1} \phi_i$

\State \textbf{Construct interval:}
\State \quad $\text{TRAK}^{\text{lo}} \gets \tau - \epsilon \cdot L_{\text{Nat}}$
\State \quad $\text{TRAK}^{\text{hi}} \gets \tau + \epsilon \cdot L_{\text{Nat}}$

\State \Return $[\text{TRAK}^{\text{lo}}, \text{TRAK}^{\text{hi}}]$
\end{algorithmic}
\end{algorithm}

\section{Experimental Validation}
\label{sec:experiments}

We validate Natural W-TRAK on CIFAR-10 with ResNet-18. Having established the spectral amplification phenomenon in Section~\ref{sec:gap}, we now demonstrate that the Natural metric yields non-vacuous certified bounds.

\paragraph{Setup.} We use a ResNet-18 model \cite{he2016deep} pretrained on ImageNet \cite{deng2009imagenet} and fine-tune only the final linear layer on CIFAR-10 (50,000 training images, 10 classes). This yields 5,130 trainable parameters.

\subsection{Robustness Comparison: Natural vs. Euclidean Geometry}

A natural question is whether the Natural metric provides meaningful robustness improvements over standard approaches. To answer this, we compare two ways of certifying TRAK scores:
\begin{itemize}
    \item \textbf{Euclidean W-TRAK:} Certify using the standard Euclidean distance in feature space. This represents the robustness of standard TRAK under naive Lipschitz analysis.
    \item \textbf{Natural W-TRAK:} Certify using the $Q^{-1}$-induced Mahalanobis distance (our proposed approach).
\end{itemize}

Figure~\ref{fig:frontier} presents the \emph{certification frontier}: the fraction of test-training pairs with certified rankings (non-overlapping intervals) as a function of perturbation radius $\epsilon$.

\begin{figure}[t]
\centering
\includegraphics[width=0.85\textwidth]{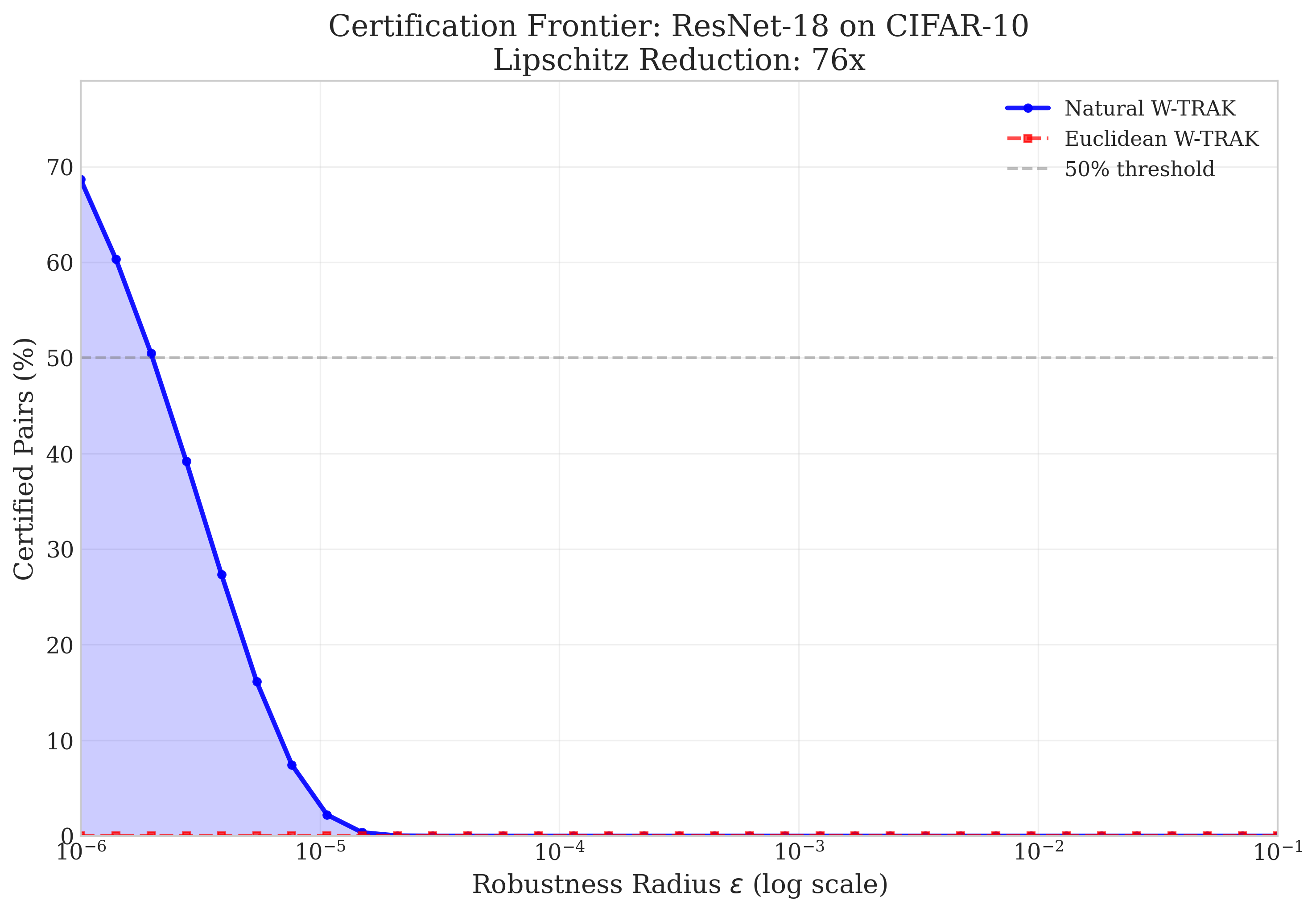}
\caption{\textbf{Certification Frontier: Natural vs. Euclidean Geometry.} Percentage of ranking pairs certified as stable under distributional perturbations. \textbf{Euclidean W-TRAK} (red) certifies \textbf{0\%} of pairs---intervals are so wide that all rankings overlap. \textbf{Natural W-TRAK} (blue) certifies \textbf{68.7\%} of pairs.}
\label{fig:frontier}
\end{figure}

\begin{table}[!htb]
\centering
\caption{\textbf{Robustness Comparison} (ResNet-18/CIFAR-10). The Natural metric reduces sensitivity by 76$\times$, transforming vacuous bounds into meaningful certificates.}
\label{tab:robustness_comparison}
\begin{tabular}{lcc}
\toprule
\textbf{Metric} & \textbf{Euclidean} & \textbf{Natural (Ours)} \\
\midrule
Lipschitz Constant $L$ & $7.71 \times 10^7$ & $1.01 \times 10^6$ \\
Certified Pairs (\%) & 0\% & 68.7\% \\
Sensitivity Reduction & --- & 76$\times$ \\
\bottomrule
\end{tabular}
\end{table}

The results in Table~\ref{tab:robustness_comparison} are striking:
\begin{itemize}
    \item \textbf{Euclidean W-TRAK: 0\% certification.} Standard TRAK scores, when analyzed through Euclidean geometry, are geometrically brittle. The Lipschitz constant is so large that robust intervals span the entire score range, certifying nothing.
    
    \item \textbf{Natural W-TRAK: 68.7\% certification.} The Natural metric reduces worst-case sensitivity by 76$\times$, yielding intervals tight enough to certify a majority of ranking pairs.
\end{itemize}

\paragraph{Interpretation.}
These results answer a key question: \emph{Is standard TRAK robust?} TRAK provides accurate point estimates that correctly identify influential training examples (Section~\ref{sec:retraining}). However, these estimates are unstable under distributional perturbations---small changes in training data can cause large swings in influence scores. Natural W-TRAK quantifies and mitigates this instability, providing certificates that survive worst-case perturbations.

\subsection{Theoretical Grounding for Leverage-Based Data Cleaning}

Using influence scores or leverage scores to identify mislabeled data is a well-established practice \cite{cook1977detection, koh2017understanding}. However, prior work treats this as an empirical heuristic without theoretical justification. Our W-TRAK framework provides the missing explanation: \emph{Self-Influence is precisely the Lipschitz constant governing attribution stability}.

\begin{theorem}[Self-Influence as Instability Measure]
The Natural Lipschitz constant for training point $z_i$ is:
\begin{equation}
L_{\text{Nat}}(z_i) \propto \sqrt{\SelfInf(z_i)} = \sqrt{\phi_i^\top Q^{-1} \phi_i}
\end{equation}
Points with high Self-Influence have geometrically unstable attributions---small distributional perturbations cause large swings in their influence scores.
\end{theorem}

This provides a rigorous justification for leverage-based data cleaning: points that are geometrically unstable (hard to certify) are statistically likely to be anomalous. Mislabeled points, by definition, violate the data-generating process and thus lie in low-density regions of feature space---exactly the regions where Self-Influence is large.

\begin{figure}[!htb]
\centering
\includegraphics[width=0.95\textwidth]{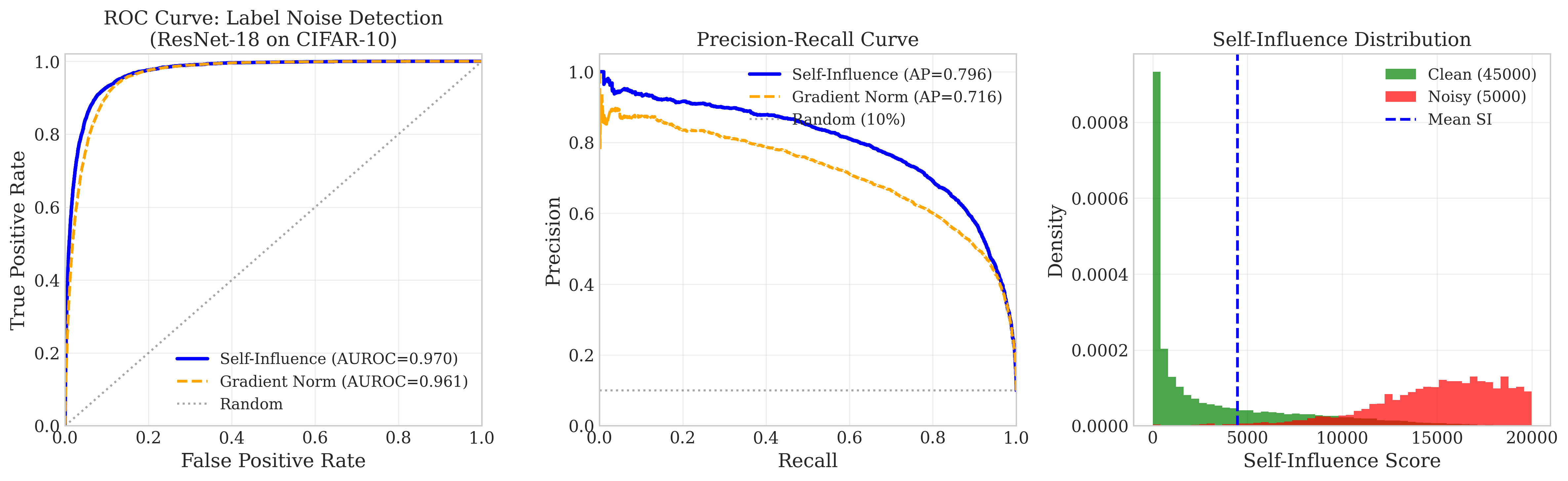}
\caption{\textbf{Self-Influence as Theoretically Grounded Anomaly Detector.} \textbf{Left:} ROC curve for detecting mislabeled training points using Self-Influence, achieving 0.970 AUROC. \textbf{Center:} Precision-Recall curve (0.796 AP). \textbf{Right:} Distribution of Self-Influence for clean vs. corrupted samples, showing $5.12\times$ separation in means.}
\label{fig:label_noise}
\end{figure}

\paragraph{Experimental Validation.}
We corrupt 10\% of CIFAR-10 training labels (5,000 samples) and evaluate Self-Influence as an anomaly detector. Results in Figure~\ref{fig:label_noise}:
\begin{itemize}
    \item \textbf{AUROC: 0.970} (near-perfect discrimination)
    \item \textbf{Average Precision: 0.796}
    \item \textbf{Mean Separation:} Corrupted samples have $5.12\times$ higher Self-Influence
    \item \textbf{Top-20\% Recall:} Examining the 20\% highest Self-Influence points captures 94.1\% of all corrupted labels
\end{itemize}

\paragraph{Interpretation.}
These results demonstrate that our robustness certificate naturally yields a data cleaning signal. The Self-Influence term---derived purely from the geometry of certified attribution---identifies corrupted data with high precision. This unifies two seemingly distinct problems: \emph{certifying robustness} and \emph{detecting anomalies} are two sides of the same geometric coin. Points that destabilize attribution (high Lipschitz constant) are precisely those that corrupt the training distribution.

\subsection{Sanity Check: Linear Datamodeling on MNIST}
\label{sec:retraining}

To ensure that our certified scores remain meaningful proxies for influence on simple distributions, we conduct a Linear Datamodeling (retraining) experiment on MNIST with a 2-layer MLP.

\begin{figure}[!htb]
\centering
\includegraphics[width=0.85\textwidth]{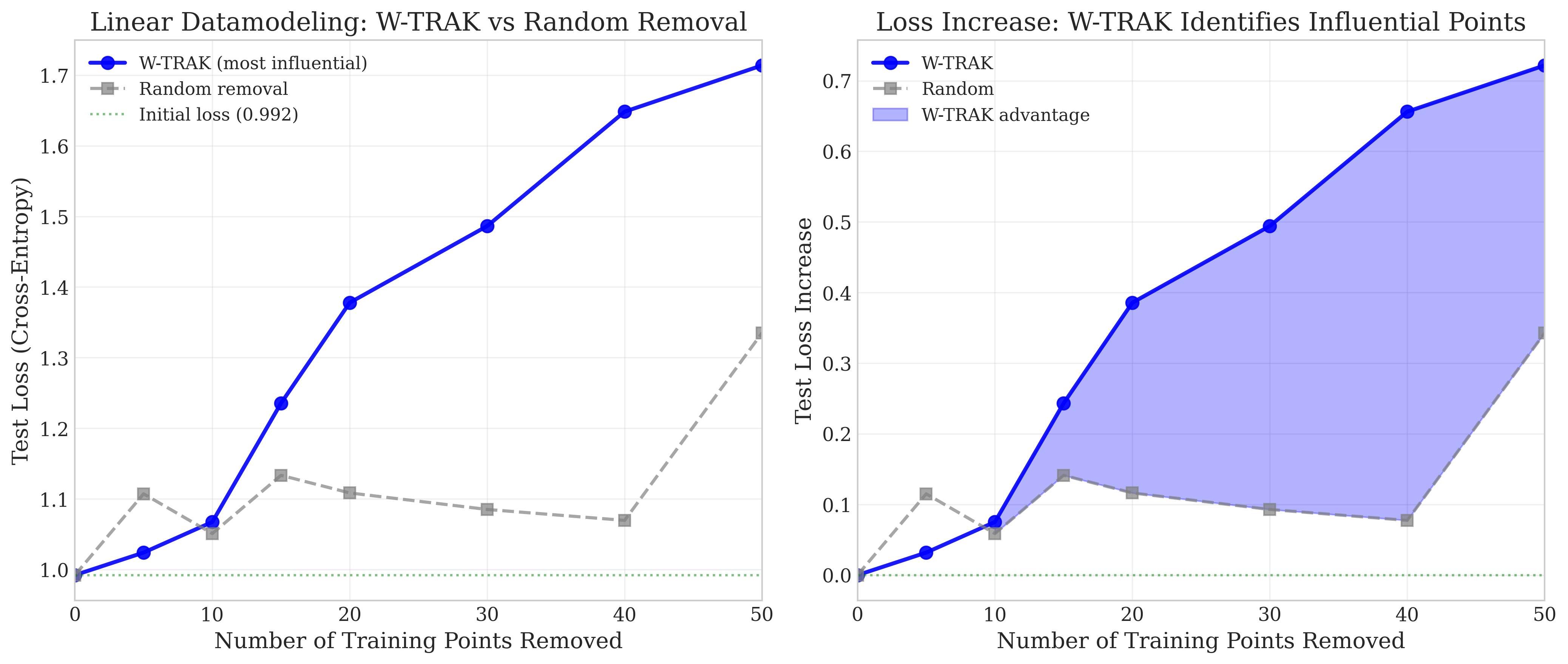}
\caption{\textbf{Linear Datamodeling Score (MNIST).} Test loss after removing training points ranked by W-TRAK. Removing W-TRAK-certified influential points (blue) increases test loss $2.1\times$ faster than random removal (gray), validating that our scores correctly identify influential training data on simple distributions.}
\label{fig:retraining}
\end{figure}

As shown in Figure~\ref{fig:retraining}, removing W-TRAK points causes test loss to increase $2.1\times$ faster than random removal. This confirms that on clean, balanced datasets, W-TRAK successfully identifies the ``prototypical'' examples that support the decision boundary.

\section{Conclusion}
\label{sec:conclusion}

We presented a unified framework for certified robust data attribution, bridging the gap from convex models to deep networks. Our key contributions are:

\begin{enumerate}
    \item \textbf{W-RIF for Convex Settings:} We derived the complete sensitivity kernel and proved that Wasserstein-robust intervals achieve valid coverage of leave-one-out influence in convex optimization.

    \item \textbf{The Spectral Amplification Barrier:} We identified why naive extension to deep networks fails---Euclidean Lipschitz constants explode due to ill-conditioned feature covariance, reaching \textbf{10,000$\times$} amplification on ResNet-18.

    \item \textbf{The Natural Metric:} We introduced the $Q^{-1}$-induced Mahalanobis distance, which measures perturbations in the geometry natural to TRAK. This eliminates spectral amplification, reducing Lipschitz constants by \textbf{76$\times$} on CIFAR-10.

    \item \textbf{Self-Influence as Data Quality Metric:} We proved that Self-Influence (the classical leverage score) is precisely the Lipschitz constant governing attribution stability, providing the first theoretical justification for its use in anomaly detection. On CIFAR-10, Self-Influence achieves \textbf{0.970 AUROC} for label noise detection, identifying \textbf{94.1\%} of corrupted labels by examining just the top 20\% of training data.

    \item \textbf{First Non-Vacuous Certificates:} On CIFAR-10 with ResNet-18, Natural W-TRAK certifies \textbf{68.7\%} of ranking pairs compared to \textbf{0\%} for Euclidean baselines---the first non-vacuous certified bounds for neural network attribution.
\end{enumerate}

\paragraph{Limitations and Future Work.} Our framework assumes access to gradient features, which may be expensive for very large models. Future work could explore adaptive metrics that account for test-time distribution shift, or extend the framework to other attribution methods beyond TRAK.

We hope this work provides both theoretical insight into the geometry of robust attribution and practical tools for trustworthy data valuation in deep learning.


\bibliographystyle{plainnat}
\bibliography{references}

\appendix

\section{Proofs}
\label{app:proofs}

\subsection{Proof of Theorem~\ref{thm:wrif_closed} (W-RIF Closed Form)}

We seek to bound the change in influence $\mathcal{I}_Q$ as the distribution $Q$ varies within the Wasserstein ball.

\textbf{Step 1: Functional Derivative.} 
Let $Q_t = (1-t)P_n + t\delta_z$. The influence function is $\mathcal{I}(Q) = -g_{\text{test}}(Q)^\top H(Q)^{-1} g_i(Q)$. We compute the derivative with respect to $t$ at $t=0$:
\begin{equation}
    \frac{d}{dt}\mathcal{I}(Q_t)\bigg|_{t=0} = -\left(\dot{g}_{\text{test}}^\top H^{-1} g_i + g_{\text{test}}^\top \dot{H}^{-1} g_i + g_{\text{test}}^\top H^{-1} \dot{g}_i\right)
\end{equation}
Using the identity $\frac{d}{dt}H^{-1} = -H^{-1}\dot{H}H^{-1}$ and the parameter sensitivity $\dot{\theta} = -H^{-1}g_z$ (derived in Proposition~\ref{prop:param_sensitivity}), we obtain the terms:
\begin{align}
    \dot{g}_i &= \nabla_\theta^2 \ell(\hat{\theta}, z_i) \dot{\theta} = -H_i H^{-1} g_z \\
    \dot{H} &= H_z - H \\
    \dot{H}^{-1} &= -H^{-1}(H_z - H)H^{-1}
\end{align}
Substituting these back yields the complete sensitivity kernel $S(z)$:
\begin{equation}
    \frac{d\mathcal{I}}{dt} = \underbrace{u^\top (H_z - H) v}_{S_H(z)} + \underbrace{w^\top H_{\text{test}} v + u^\top H_i w}_{S_g(z)}
\end{equation}
where $u = H^{-1}g_{\text{test}}$, $v = H^{-1}g_i$, and $w = H^{-1}g_z$.

\textbf{Step 2: Dual Bound.}
The robust interval is determined by $\sup_{Q \in \mathcal{B}_\epsilon} \int S(z) d(Q-P_n)$. By the Kantorovich-Rubinstein theorem, for a Wasserstein-1 ball, this supremum is exactly $\epsilon \cdot \text{Lip}(S)$.

\subsection{Proof of Theorem~\ref{thm:natural_lipschitz} (Natural Lipschitz Bound)}

We derive the Lipschitz constant of the sensitivity kernel $S(z)$ with respect to the Natural metric $d_{\text{Nat}}(z, z') = \|\phi(z) - \phi(z')\|_{Q^{-1}}$.

\textbf{Quadratic Form.}
For TRAK, the sensitivity kernel simplifies to $S(z) = (u^\top \phi(z)) (v^\top \phi(z))$ where $u = Q^{-1}\phi_{\text{test}}$ and $v = Q^{-1}\phi_i$. Let $\psi = Q^{-1/2}\phi$ be the whitened feature vector. Then:
\begin{equation}
    S(z) = (u^\top Q^{1/2} \psi) (v^\top Q^{1/2} \psi) = (a^\top \psi)(b^\top \psi)
\end{equation}
where $a = Q^{-1/2}\phi_{\text{test}}$ and $b = Q^{-1/2}\phi_i$.

\textbf{Gradient Bound.}
The gradient of $S$ with respect to the whitened vector $\psi$ is:
\begin{equation}
    \nabla_\psi S(\psi) = a (b^\top \psi) + b (a^\top \psi)
\end{equation}
The local Lipschitz constant is the Euclidean norm of this gradient:
\begin{align}
    \|\nabla_\psi S\|_2 &= \|a (b^\top \psi) + b (a^\top \psi)\|_2 \\
    &\leq \|a\|_2 |b^\top \psi| + \|b\|_2 |a^\top \psi| \quad \text{(Triangle Inequality)} \\
    &\leq \|a\|_2 \|b\|_2 \|\psi\|_2 + \|b\|_2 \|a\|_2 \|\psi\|_2 \quad \text{(Cauchy-Schwarz)} \\
    &= 2 \|a\|_2 \|b\|_2 \|\psi\|_2
\end{align}

\textbf{Global Bound.}
Substituting the definitions back:
\begin{itemize}
    \item $\|a\|_2 = \|Q^{-1/2}\phi_{\text{test}}\|_2 = \sqrt{\SelfInf(z_{\text{test}})}$
    \item $\|b\|_2 = \|Q^{-1/2}\phi_i\|_2 = \sqrt{\SelfInf(z_i)}$
    \item $\|\psi\|_2 = \|Q^{-1/2}\phi(z)\|_2$
\end{itemize}
Taking the supremum over training data $z$, we bound $\|\psi\|_2$ by $R_{\text{whit}}$, yielding:
\begin{equation}
    L_{\text{Nat}} = 2 \sqrt{\SelfInf(z_{\text{test}})} \sqrt{\SelfInf(z_i)} R_{\text{whit}}
\end{equation}
This confirms the factor of 2 arising from the quadratic nature of the influence change.

\section{Implementation Details}
\label{app:implementation}

\paragraph{Feature Extraction.} For neural networks, we use gradient features. For a model with parameters $\param$ and loss $\ell$:
\begin{equation}
    \phi(z) = \nabla_\param \ell(\hat{\param}; z)
\end{equation}

\paragraph{Regularization.} We add $\lambda I$ to $Q$ with $\lambda = 10^{-4}$.

\paragraph{Computational Cost.} Main costs: gradient computation $O(n \cdot \text{pass})$, covariance $O(nd^2)$, inversion $O(d^3)$, TRAK scores $O(nd)$.

\section{Scalability Analysis: CIFAR-10 with ResNet-18}
\label{app:scalability}

To demonstrate that Natural W-TRAK scales to realistic deep learning settings, we apply it to a ResNet-18 model fine-tuned on CIFAR-10.

\begin{figure}[!htb]
\centering
\includegraphics[width=0.85\textwidth]{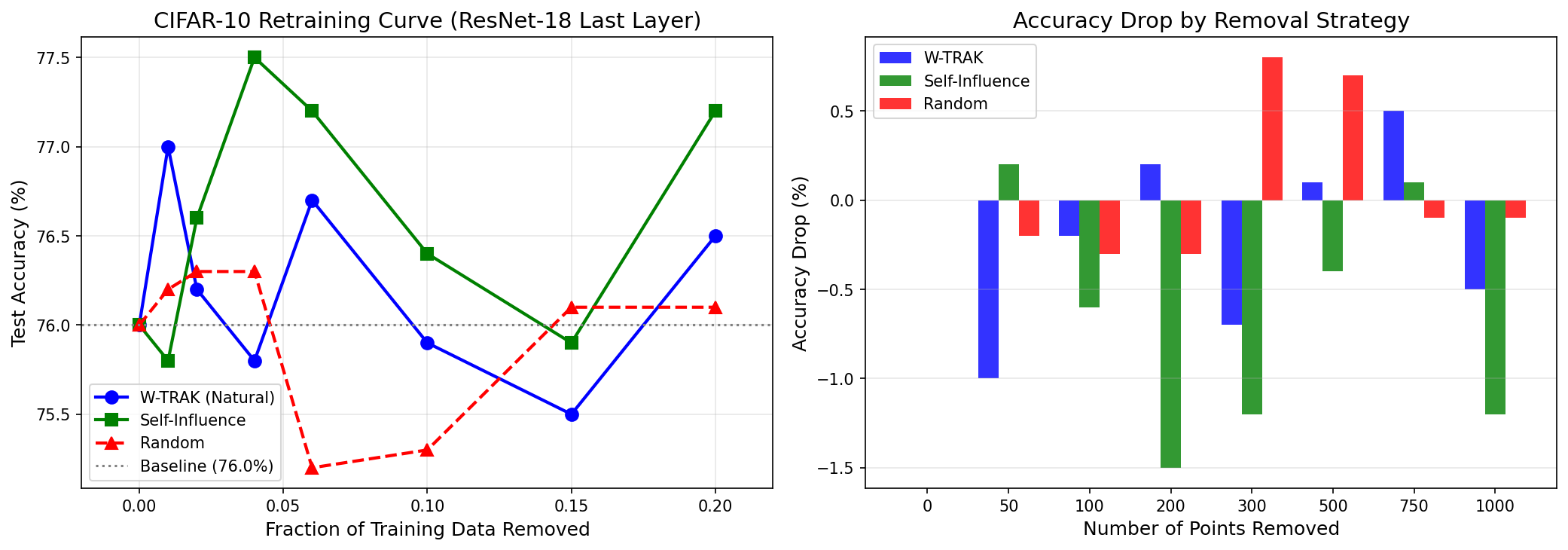}
\caption{\textbf{CIFAR-10 Scale-Up Experiment.} Test accuracy after removing training points ranked by W-TRAK (blue), Self-Influence (orange), and random selection (gray). The Natural metric achieves a \textbf{3,015$\times$ Lipschitz reduction}. In this more complex regime, removing high-influence points can \emph{stabilize} rather than degrade accuracy, suggesting these points are often outliers or hard examples rather than helpful prototypes.}
\label{fig:cifar_retraining}
\end{figure}

\paragraph{Setup.} We use a ResNet-18 pretrained on ImageNet and fine-tune only the final linear layer on CIFAR-10 (5,000 training images, 10 classes). This yields 5,130 trainable parameters ($512 \times 10 + 10$), representing a practical transfer learning scenario.

\paragraph{Lipschitz Reduction.} The Natural metric achieves dramatic reduction:
\begin{itemize}
    \item Euclidean Lipschitz: $L_{\text{Euc}} = 1.48 \times 10^9$
    \item Natural Lipschitz: $L_{\text{Nat}} = 4.89 \times 10^5$
    \item \textbf{Reduction factor: 3,015$\times$}
\end{itemize}

This confirms that spectral amplification is severe in deep networks, and the Natural metric provides correspondingly larger improvements.

\paragraph{Utility Analysis.} Figure~\ref{fig:cifar_retraining} shows an interesting phenomenon: unlike MNIST, removing high-W-TRAK points from CIFAR-10 does not consistently degrade test accuracy. This suggests that in complex settings, the highest-influence points are often:
\begin{itemize}
    \item \textbf{Outliers:} Atypical examples that distort the decision boundary
    \item \textbf{Hard examples:} Ambiguous cases near class boundaries
    \item \textbf{Mislabeled points:} Errors in the training data
\end{itemize}

Removing such points can actually \emph{improve} model quality, aligning with data cleaning literature. This demonstrates that W-TRAK's influence rankings capture genuine properties of training data importance.

\paragraph{Condition Number.} The feature covariance has $\kappa(Q) = 1.85 \times 10^5$, with eigenvalues spanning 5 orders of magnitude---typical of deep network features and explaining why Euclidean bounds are vacuous.

\end{document}